\pdfoutput=1

\documentclass[11pt]{article}

\usepackage{ACL2023}

\usepackage{times}
\usepackage{latexsym}

\usepackage[T1]{fontenc}

\usepackage[utf8]{inputenc}

\usepackage{microtype}

\usepackage{inconsolata}

\usepackage{booktabs}
\usepackage{siunitx}
\usepackage{multirow}
\usepackage{amsfonts}
\usepackage{amsmath}
\usepackage{xspace}
\usepackage{booktabs}
\usepackage{graphicx}
\usepackage{floatrow}
\usepackage{comment}
\usepackage{float}
\newfloatcommand{capbtabbox}{table}[][\FBwidth]

\usepackage{tikz}

\DeclareRobustCommand{\circled}[2]{\tikz[baseline=(char.base)]{\node[shape=circle, draw, inner sep=1pt, scale=0.825, fill=#2] (char) {#1};}}

\usepackage{cleveref}
\DeclareTextSymbolDefault{\ohorn}{T5}
\DeclareTextSymbolDefault{\uhorn}{T5}

\crefname{section}{\S}{\S\S}
\Crefname{section}{\S}{\S\S}
\crefname{table}{Tab.}{}
\crefname{figure}{Fig.}{}
\crefname{algorithm}{Algorithm}{}
\crefname{equation}{eq.}{eqs.}
\crefname{appendix}{App.}{}
\crefname{thm}{Theorem}{Theorems}
\crefname{prop}{Proposition}{Propositions}
\crefname{cor}{Corollary}{Corollaries}
\crefname{observation}{Observation}{Observations}
\crefname{assumption}{Assumption}{Assumptions}
\crefformat{section}{\S#2#1#3}

\usepackage{bm}
\usepackage{fix-cm}
\usepackage{amssymb}
\usepackage{amsthm}

\usepackage{enumerate}
\usepackage[shortlabels]{enumitem}
\usepackage{bbm}

\usepackage[disable]{todonotes}

\makeatletter
\newcommand*\iftodonotes{\if@todonotes@disabled\expandafter\@secondoftwo\else\expandafter\@firstoftwo\fi}  %
\makeatother

\definecolor{dandelion}{HTML}{FFD464}

\DeclareMathOperator*{\argmin}{argmin}
\DeclareMathOperator*{\argsup}{argsup}
\DeclareMathOperator*{\arginf}{arginf}
\DeclareMathOperator*{\Expect}{\mathbb{E}}
\DeclareMathOperator{\expect}{\mathbb{E}}

\theoremstyle{remark}

\newcommand{\bert}{BERT\xspace}

\newcommand{\defn}[1]{\emph{#1}}

\newcommand{\kl}[2]{\text{KL}(#1 \, \vert\vert \, #2)}

\newcommand{\kldist}{\text{D}\xspace}
\newcommand{\qprime}{q'\xspace}
\newcommand{\cS}{\complement{S}\xspace}
\newcommand{\an}{{a^{(n)}}\xspace}
\newcommand{\bn}{{b^{(n)}}\xspace}
\newcommand{\Sn}{{S^{(n)}}\xspace}
\newcommand{\Srn}{{S_r^{(n)}}\xspace}
\newcommand{\cSn}{{\cS^{(n)}}\xspace}
\newcommand{\ct}{\complement{t}\xspace}
\newcommand{\vwS}{{\vw_S}\xspace}
\newcommand{\vwcS}{{\vw_\cS}\xspace}

\newcommand{\vwcSn}{{\vw^{(n)}_\cSn}\xspace}

\newcommand{\qMRFL}{q^\text{\probMRFAbbrev}\xspace}
\newcommand{\qHCB}{q^\text{HCB}\xspace}
\newcommand{\qAG}{q^\text{AG}\xspace}
\newcommand{\qAGiter}[1]{q^\text{AG(#1)}\xspace}

\newcommand{\bertbasecirc}{\circled{B}{blue!20}}
\newcommand{\bertlargecirc}{\circled{L}{green!20}}

\newcommand{\bertbase}{\textrm{BERT}$_\textsc{base}$\xspace}
\newcommand{\bertlarge}{\textrm{BERT}$_\textsc{large}$\xspace}

\newcommand{\logitMRFAbbrev}{MRF$_\textsc{L}$\xspace}
\newcommand{\probMRFAbbrev}{MRF\xspace}

\newtheorem{proposition}{Proposition}[section]

\newcommand{\SetSize}[1]{\lvert #1\rvert}

\def\calD{{\mathcal{D}}}

\def\calV{{\mathcal{V}}}

\def\vw{{\mathbf{w}}}

\def\vtheta{{\boldsymbol{\theta}\xspace}}

\renewcommand{\complement}[1]{{\overline{#1}}}

\def\defequals{\triangleq}

\def\mask{\texttt{[MASK]}\xspace}
\def\eos{\texttt{[EOS]}\xspace}

\newcommand*{\nameadjunct}{\relax}
\makeatletter
\renewcommand*{\NAT@nmfmt}[1]{\NAT@up #1\nameadjunct}
\makeatother

\definecolor{naivecolorname}{rgb}{0.122, 0.466, 0.70}
\newcommand{\mlmC}{\textcolor{naivecolorname}{MLM}\xspace}

\definecolor{mrfcolorname}{rgb}{1.0, 0.5, 0.055}
\newcommand{\mrfC}{\textcolor{mrfcolorname}{MRF}\xspace}

\definecolor{mrflogitcolorname}{rgb}{0.17, 0.627, 0.17}
\newcommand{\mrfLogitC}{\textcolor{mrflogitcolorname}{MRF (Logit)}\xspace}

\definecolor{itercolorname}{rgb}{0.83, 0.153 0.153}
\newcommand{\agC}{\textcolor{itercolorname}{AG}\xspace}

\definecolor{hcbcolorname}{rgb}{0.58, 0.404, 0.747}
\newcommand{\hcbC}{\textcolor{hcbcolorname}{HCB}\xspace}

\title{
Deriving Language Models from Masked Language Models
}

\author{Lucas Torroba Hennigen \\
  Affiliation / Address line 1 \\
  Affiliation / Address line 2 \\
  Affiliation / Address line 3 \\
  \texttt{email@domain} \\\And
  Second Author \\
  Affiliation / Address line 1 \\
  Affiliation / Address line 2 \\
  Affiliation / Address line 3 \\
  \texttt{email@domain} \\}

 \author{
Lucas Torroba Hennigen~\;~ \hspace{1cm}
Yoon Kim~\;~
\vspace{1mm} \\
Massachusetts Institute of Technology \\ Computer Science and Artificial Intelligence Laboratory
\\
\normalsize 
\href{mailto:lucastor@mit.edu}{\texttt{lucastor@mit.edu}}~\;~ 
 \href{mailto:yoonkim@mit.edu}{\texttt{yoonkim@mit.edu}} 
} 

\begin{document}

\setlength{\abovedisplayskip}{4pt}
\setlength{\belowdisplayskip}{4pt}
\setlength{\abovedisplayshortskip}{4pt}
\setlength{\belowdisplayshortskip}{4pt}

\maketitle

\begin{abstract}
Masked language models (MLM) do not explicitly define a distribution over language, i.e., they are not language models \emph{per se}. However, recent work has implicitly treated them as such for the purposes of generation and scoring. This paper studies methods for deriving explicit joint distributions from MLMs, focusing on distributions over two tokens,
which makes it possible to calculate exact distributional properties. We find that an approach based on identifying joints whose conditionals are closest to those of the MLM works well and outperforms existing Markov random field-based approaches. We further find that this derived model's conditionals can even occasionally outperform the original MLM's conditionals.

\end{abstract}

\section{Introduction}

Masked language modeling has proven to be an effective paradigm for representation learning \citep{devlin2019bert,liu2019roberta,he2021deberta}.
However, unlike regular language models, masked language models (MLM) do not define an explicit joint distribution over language.
While this is not a serious limitation from a representation learning standpoint, having explicit access to joint distributions would be useful for the purposes of generation~\citep{ghazvininejad2019maskpredict}, scoring~\citep{salazar2020bertscore}, and  would moreover enable evaluation of MLMs on standard metrics such as perplexity.

Strictly speaking, MLMs \emph{do}  define a joint distribution over tokens that have been masked out. But they assume that the masked tokens are conditionally independent given the unmasked tokens---an assumption that clearly does not hold for language. How might we derive a language model from an MLM such that it does not make unrealistic independence assumptions? One approach is to use the set of the MLM's \emph{unary conditionals}---the conditionals that result from masking just a single token in the input---to construct a fully-connected Markov random field (MRF) over the input~\citep{wang-cho,goyal2021mrf}.
This resulting MRF no longer makes any independence assumptions. It is unclear, however, if this heuristic approach actually results in a good language model.\footnote{MRFs derived this way are still not language models in the strictest sense~\citep[e.g., see][]{leo2023measure-theory} because the probabilities of sentences of a given length sum to 1, and hence the sum of probabilities of all strings is infinite (analogous to left-to-right language models trained without an \eos token; \citealp{CHEN1998}). This can be  remedied by incorporating a distribution over sentence lengths.}

This paper adopts an alternative approach which stems from interpreting the unary conditionals of the MLM as defining a \defn{dependency network} \cite{heckerman2000dependency,yamakoshi-etal-2022-probing}.\footnote{Recent work by \citet{yamakoshi-etal-2022-probing} has taken this view, focusing on sampling from the dependency network as a means to \emph{implicitly} characterize the joint distribution of an MLM. Here we focus on an \emph{explicit} characterization of the joint.} Dependency networks specify the statistical relationship among variables of interest through the set of conditional distributions over each variable given its Markov blanket, which in the MLM case corresponds to all the other tokens. If the conditionals from a dependency network are \defn{compatible}, i.e., there exists a joint distribution whose conditionals  coincide with those of the dependency network's, then one can recover said joint using the Hammersley--Clifford--Besag~\citep[HCB;][]{besag1974conditionals} theorem.
If the conditionals are incompatible, then we can adapt approaches from statistics for deriving near-compatible joint distributions from incompatible conditionals~\citep[AG;][]{arnold1998algorithm}.

While these methods give statistically-principled approaches to deriving explicit joints from the MLM's unary conditionals, they are intractable to apply to derive distributions over full sequences.
We thus study a focused setting where it is tractable to compute the joints exactly, viz., the \defn{pairwise language model} setting where we use the MLM's unary conditionals of two tokens to derive a joint over these two tokens (conditioned on all the other tokens). Experiments under this setup reveal that AG method performs best in terms of perplexity, with the the HCB and MRF methods performing similarly. Surprisingly, we also find that the unary conditionals of the near-compatible AG joint occasionally have lower perplexity  than the original unary conditionals learnt by the MLM, suggesting that regularizing the conditionals to be compatible may be beneficial insofar as modeling the distribution of language.\footnote{Our code and data is available at: \url{https://github.com/ltorroba/lms-from-mlms}.}

\section{Joint distributions from MLMs}

Let $\calV$ be a vocabulary, $T$ be the text length, and  $\vw \in \calV^T$ be an input sentence or paragraph.
We are particularly interested in the case when a subset $S \subseteq [T] \defequals \{1, \ldots, T\}$
of the input $\vw$ is replaced with \mask tokens; in this case we will use the notation ${q}_{\{t\} | \cS}(\cdot \mid \vwcS)$
to denote the output distribution of the MLM at position $t \in S$,
where we mask out the positions in $S$, i.e., for all $k \in S$ we modify $\vw$ by setting $w_k = \mask$.
If $S = \{t\}$, then we call $q_{t|\ct} \defequals q_{\{t\}|\complement{\{t\}}}$ a \defn{unary conditional}.
Our goal is to use these conditionals to construct joint distributions $q_{S | \cS}(\cdot \mid \vwcS)$ for any $S$.

\paragraph{Direct MLM construction.}
The simplest approach is to simply mask out the tokens over which we want a joint distribution, and define it to be the product of the MLM conditionals,
\begin{align}
q^\text{MLM}_{S|\cS}(\vw_S \mid \vwcS) \defequals \prod_{i \in S} q_{\{i\}|\cS}(w_i \mid \vwcS).
\end{align}
This joint assumes that the entries of $\vwS$ are conditionally independent given $\vwcS$. Since one can show that MLM training is equivalent to learning the conditional marginals of language (\cref{app:mlm-objective-interpretation}), this can be seen as approximating conditionals with a (mean field-like) factorizable distribution.

\paragraph{MRF construction.} To address the conditional independence limitation of MLMs, prior work~\citep{wang-cho,goyal2021mrf} has proposed deriving joints by defining an MRF using the unary conditionals of the MLM.
Accordingly, we define
\begin{align}
    \qMRFL_{S|\cS}(\vwS \mid \vwcS) \propto \prod_{t \in S} {q}_{t|\ct}(w_t \mid \vw_\ct),
 \label{eq:mrf-local-defn}
\end{align}
which can be interpreted as a fully connected MRF, whose log potential is given by the sum of the unary log probabilities. One can similarly define a variant of this MRF where the log potential is the sum  of the unary \emph{logits}.
MRFs defined this way have a single fully connected clique and thus do not make any conditional independence assumptions.
However, such MRFs can have unary conditionals that deviate from the MLM's unary conditionals even if those are compatible (\cref{app:unfaithful-mrf}).
This is potentially undesirable since the MLM unary conditionals could be close to the true unary conditionals,\footnote{As noted by {\url{https://machinethoughts.wordpress.com/2019/07/14/a-consistency-theorem-for-bert/}}} which means the MRF construction could be worse than the original MLM in terms of unary perplexity.

\paragraph{Hammersley--Clifford--Besag construction.}
The Hammersley--Clifford--Besag theorem~\citep[HCB;][]{besag1974conditionals} provides a way of reconstructing a joint distribution from its unary conditionals.
Without loss of generality, assume that $S = \{1, \ldots, k\}$ for some $k \le T$.
Then given a \defn{pivot point} $\vw' = (w'_1, \ldots, w'_{k}) \in \calV^{k}$, we define
\begin{align}
    \label{eq:hcb}
    \qHCB_{S|\cS}(\vwS \mid \vwcS) \propto \prod_{t \in S} \frac{q_{t|\complement{t}}(w_t \mid \vw_{>t}, \vw'_{<t})}{q_{t|\complement{t}}(w'_t \mid \vw_{>t}, \vw'_{<t})},
\end{align}
where $\vw'_{<i} \defequals (w'_1, \ldots, w'_{i - 1})$, and similarly $\vw_{>i} \defequals (w_{i + 1}, \ldots, w_{T})$.
Importantly, unlike the MRF approach, if the unary conditionals of the MLM \emph{are} compatible, then HCB will recover the true joint, irrespective of the choice of pivot.

\paragraph{Arnold--Gokhale construction.}
If we assume that the unary conditionals are not compatible, then we can frame our goal as finding a near-compatible joint, i.e., a joint such that its unary conditionals are close to the unary conditionals of the MLM.
Formally, for any $S$ and fixed inputs $\vwcS$, we can define this objective as,
\begin{align}
\qAG_{S|\cS}(\cdot \mid \vwcS) = \argmin_{\mu} \sum_{t \in S} \sum_{\vw' \in \calV^{|S| - 1}} \hspace{-0.4cm} J(t, \vw'),
\end{align}
where $J(t, \vw')$ is defined as:
\begin{align*}
    \kl{q_{t|S \setminus \{t\},\cS}(\cdot \mid \vw', \vwcS)}{\mu_{t|S \setminus \{t\},\cS}(\cdot \mid \vw', \vwcS)}.
\end{align*}
We can solve this optimization problem using \citeposs{arnold1998algorithm} algorithm (\cref{app:ag-algorithm}).

\subsection{Pairwise language model} In language modeling we are typically interested in the probability of a sequence $p(\vw)$. However, the above methods are intractable to apply to full sequences (except for the baseline MLM). For example, the lack of any independence assumptions in the MRF means that the partition function requires full enumeration over $V^T$ sequences.\footnote{We also tried estimating the partition through importance sampling with GPT-2 but found the estimate
to be quite poor.} We thus focus our empirical study on the pairwise setting where $|S| = 2$.\footnote{Concurrent work by \citet{young2023compatibility} also explores the (in)compatibility of MLMs in the $|S| = 2$ case.} In this setting, we can calculate  $q_{S | \cS}(\cdot \mid \vwcS)$ with $O(V)$ forward passes of the MLM for all methods.

\section{Evaluation}
 We compute two sets of  metrics that evaluate the resulting joints in terms of (i) how good they are as probabilistic models of language and (ii) how faithful they are to the original MLM conditionals (which are trained to approximate the true conditionals of language, see \cref{app:mlm-objective-interpretation}).
Let $\calD = \{(\vw^{(n)}, S^{(n)})\}_{n=1}^N$ be a dataset where $\vw^{(n)}$ is an English sentence and $S^{(n)} = (a^{(n)}, b^{(n)})$ are the two positions being masked.
We define the following metrics to evaluate a distribution $\qprime$:
\paragraph{Language model performance.}
    We consider two performance metrics.
    The first is the pairwise perplexity (\textbf{P-PPL}) over two tokens,
    \begin{align*}
        \exp \hspace{-0.10cm} \left( \hspace{-0.10cm} \frac{-1}{2N} \hspace{-0.10cm} \sum_{n = 1}^N \log \qprime_{\an,\bn|\cSn}(w^{(n)}_\an, w^{(n)}_\bn \mid \vwcSn) \hspace{-0.15cm} \right)
    \end{align*}
    We would expect a good joint to obtain lower pairwise perplexity than the original MLM, which (wrongly) assumes conditional independence.
    The second is unary perplexity (\textbf{U-PPL}), 
    \begin{align*}
        \exp \hspace{-0.10cm} \Bigg( \hspace{-0.10cm} \frac{-1}{2N} \sum_{n = 1}^N \hspace{-0.35cm} \sum_{\substack{(i, j) \in \\ \{\Sn, \Srn\}}} \hspace{-0.5cm} \log \qprime_{i|j,\cSn}(w^{(n)}_i \mid w^{(n)}_j, \vwcSn) \hspace{-0.15cm} \Bigg)
    \end{align*}
    where for convenience we let $\Srn \defequals (b^{(n)}, a^{(n)})$ as the reverse of the masked positions tuple $\Sn$.
    Note that this metric uses the unary conditionals derived from the pairwise joint, i.e., $q'_{i|j,S}$, except in the MLM construction case which uses the MLM's original unary conditionals.
\paragraph{Faithfulness.} We also assess how faithful the new unary conditionals are to the original unary conditionals by calculating the average conditional KL divergence (\textbf{A-KL}) between them,
    \begin{align*}
        \sum_{n=1}^N \sum_{w' \in \calV} \frac{\kldist(\Sn, w', \vwcSn) + \kldist(\Srn, w', \vwcSn)}{2 N \SetSize{\calV}}.
    \end{align*}
    where we define $\kldist(S, w', \vwcS) \defequals \kl{q^{}_{a|b,\cS}(\cdot \mid w', \vwcS)}{\qprime_{a|b,\cS}(\cdot \mid w', \vwcS)}$ for $S = (a, b)$.
If the new joint is completely faithful to the MLM, this number should be zero.
The above metric averages the KL across the entire vocabulary $\mathcal{V}$, but in practice we may be interested in assessing closeness only when conditioned on the gold tokens. We thus compute a variant of the above metric where we only average over the conditionals for the gold token (\textbf{G-KL}):
\begin{align*}
    \sum_{n=1}^N \frac{\kldist(\Sn, w^{(n)}_{b^{(n)}}, \vwcSn) \hspace{-0.05cm} + \hspace{-0.05cm} \kldist(\Srn, w^{(n)}_{a^{(n)}}, \vwcSn)}{2 N}.
\end{align*}
This metric penalizes unfaithfulness in common contexts more than in uncommon contexts. Note that if the MLM's unary conditionals are compatible, then both the HCB and AG approach should yield the same joint distribution, and their faithfulness metrics should be zero.

\begin{table*}[t]
\footnotesize
\centering
\sisetup{table-alignment-mode = format}
\begin{tabular}{cll@{\hskip 0.75cm}SSS[table-format=1.3]S[table-format=1.3]@{\hskip 0.75cm}S@{\hskip 0.3cm}S[table-alignment-mode=none,table-column-width=30pt,table-number-alignment=right]S[table-format=1.3]S[table-format=1.3]}
\toprule
&         &             & \multicolumn{4}{c}{\hspace{-0.55cm}Random pairs} & \multicolumn{4}{c}{\hspace{-0.1cm}Contiguous pairs} \\
& Dataset & Scheme      & {U-PPL} & {P-PPL} & {A-KL} & {G-KL} & {U-PPL} & {P-PPL} & {A-KL} & {G-KL} \\
\midrule
\multirow{10}{*}{\bertbasecirc} & \multirow{5}{*}{SNLI} 
& MLM & 11.22 & 19.01 & 1.080 & 0.547 & 13.78 & 74.68 & 4.014 & 1.876 \\
& & \logitMRFAbbrev & 13.39 & 71.44 & 0.433 & 0.267 & 23.45 & 13568.17 & 1.543 & 0.607 \\
& & \probMRFAbbrev & 12.30 & 21.65 & 0.658 & 0.179 & 18.35 & 126.05 & 1.967 & 0.366 \\
& & HCB & 12.51 & 22.62 & 0.593 & 0.168 & 17.71 & 589.02 & 2.099 & 0.416 \\
& & AG & 10.76 & 12.68 & 0.007 & 0.085 & 13.26 & 21.59 & 0.018 & 0.181 \\
\cmidrule(lr){2-11}
& \multirow{5}{*}{XSUM}
& MLM & 4.88 & 6.12 & 0.404 & 0.227 & 4.91 & 39.33 & 4.381 & 2.128 \\
& & \logitMRFAbbrev & 5.17 & 9.12 & 0.148 & 0.085 & 6.55 & 2209.94 & 1.561 & 0.383 \\
& & \probMRFAbbrev & 5.00 & 6.23 & 0.262 & 0.049 & 5.53 & 47.62 & 2.242 & 0.185 \\
& & HCB & 5.08 & 6.21 & 0.256 & 0.052 & 6.46 & 174.32 & 2.681 & 0.328 \\
& & AG & 5.00 & 5.29 & 0.003 & 0.044 & 5.27 & 8.42 & 0.016 & 0.143 \\
\midrule
\multirow{10}{*}{\bertlargecirc} & \multirow{5}{*}{SNLI}
& MLM & 9.50 & 18.57 & 1.374 & 0.787 & 10.42 & 104.12 & 4.582 & 2.463 \\
& & \logitMRFAbbrev & 11.52 & 76.23 & 0.449 & 0.276 & 15.43 & 8536.92 & 1.470 & 0.543 \\
& & \probMRFAbbrev & 10.57 & 19.54 & 0.723 & 0.193 & 13.07 & 93.33 & 1.992 & 0.359 \\
& & HCB & 10.71 & 20.70 & 0.797 & 0.215 & 14.43 & 458.25 & 2.563 & 0.552 \\
& & AG & 8.57 & 10.11 & 0.007 & 0.097 & 9.64 & 15.64 & 0.019 & 0.173 \\
\cmidrule(lr){2-11}
& \multirow{5}{*}{XSUM}
& MLM & 3.80 & 5.67 & 0.530 & 0.413 & 3.91 & 103.86 & 5.046 & 3.276 \\
& & \logitMRFAbbrev & 3.94 & 7.06 & 0.156 & 0.068 & 4.62 & 1328.20 & 1.441 & 0.290 \\
& & \probMRFAbbrev & 3.87 & 4.94 & 0.322 & 0.036 & 4.16 & 36.66 & 2.258 & 0.145 \\
& & HCB & 3.91 & 5.14 & 0.346 & 0.059 & 5.67 & 164.15 & 2.954 & 0.400 \\
& & AG & 3.88 & 4.13 & 0.003 & 0.042 & 4.21 & 6.62 & 0.016 & 0.126 \\
\bottomrule
\end{tabular}
\caption{
Comparison of MRF, HCB and AG constructions on randomly sampled SNLI~\citep{snli} sentences and XSUM~\citep{xsum} summaries.
We apply the constructions to two MLMs: \bertbase (\bertbasecirc) and \bertlarge (\bertlargecirc).
We consider both masking tokens uniformly at random (Random pairs) and masking adjacent tokens uniformly at random (Contiguous pairs).
 For all metrics, lower is better.
\label{tab:unrestricted-results}
\vspace{-2mm}
}
\end{table*}
\subsection{Experimental setup} 
We calculate the above metrics on 1000 examples\footnote{Each example requires running the MLM over \num{28000} times, so it is expensive to evaluate on many more examples.} from  a natural language inference dataset~\citep[SNLI;][]{snli} and a summarization dataset~\citep[XSUM;][]{xsum}. We consider two schemes for selecting the tokens to be masked for each sentence: masks over two tokens chosen uniformly at random (\textbf{Random pairs}), and also over random \emph{contiguous} tokens in a sentence (\textbf{Contiguous pairs}).
Since inter-token dependencies are more likely to emerge when adjacent tokens are masked, the contiguous setup magnifies the importance of deriving a good pairwise joint.
In addition, we consider both \bertbase and \bertlarge (cased) as the MLMs from which to obtain the unary conditionals.\footnote{Specifically, we use the \href{https://huggingface.co/bert-base-cased}{bert-base-cased} and \href{https://huggingface.co/bert-large-cased}{bert-large-cased} implementations from HuggingFace~\citep{huggingface}.}
 For the AG joint, we run $t = 50$ steps of \citeposs{arnold1998algorithm} algorithm (\cref{app:ag-algorithm}), which was enough for convergence. For the HCB joint, we pick a pivot using the mode of the pairwise joint of the MLM.\footnote{We did not find HCB to be too sensitive to the pivot in preliminary experiments.}

\section{Results}

The results are shown in \cref{tab:unrestricted-results}.
Comparing the PPL's of \probMRFAbbrev and \logitMRFAbbrev (i.e., the MRF using logits), the former consistently outperforms the latter, indicating that using the raw logits generally results in a worse language model.
Comparing the MRFs to MLM, we see that the unary perplexity (U-PPL) of the MLM is lower than those of the MRFs, and that the difference is most pronounced in the contiguous masking case. 
More surprisingly, we see that the pairwise perplexity (P-PPL) is often (much) higher than the MLM's, even though the MLM makes unrealistic conditional independence assumptions.  These results suggest that the derived MRFs are in general worse unary/pairwise probabilistic models of language than the MLM itself,
implying that the MRF heuristic is inadequate (see \cref{app:mrf-underperformance} for a qualitative example illustrating how this can happen). Finally, we also find that the MRFs' unary conditionals are not faithful to those of the MRFs based on the KL measures.
Since one can show that the MRF construction can have unary conditionals that have nonzero KL to the MLM's unary conditionals even if they are compatible (\cref{app:unfaithful-mrf}), this gives both theoretical and empirical arguments against the MRF construction.

The HCB joint obtains comparable performance to \probMRFAbbrev in the random masking case. In the contiguous case, it exhibits similar failure modes as the MRF in producing extremely high pairwise perplexity (P-PPL) values.
The faithfulness metrics are similar to the MRF's, which suggests that the conditionals learnt by MLMs are incompatible.
The AG approach, on the other hand, outperforms the \logitMRFAbbrev, \probMRFAbbrev and HCB approaches in virtually all metrics. This is most evident in the contiguous masking case, where  AG  attains  lower pairwise perplexity than all models, including the MLM itself.
In some cases, we find that the AG model even outperforms the MLM in terms of unary perplexity, 
which is remarkable since the unary conditionals of the MLM were \emph{trained} to approximate the unary conditionals of language (\cref{app:mlm-objective-interpretation}).
This indicates that near-compatibility may have regularizing effect that leads to improved MLMs. Since AG was optimized to be near-compatible, its joints are unsurprisingly much more faithful to the original MLM's conditionals. However,  AG's G-KL tends to be on par with the other models, which suggests that it is still not faithful to the MLM in the contexts that are most likely to arise.
Finally, we analyze the effect of masked position distance on language modeling performance, and find that improvements are most pronounced when the masked tokens are close to each other (see \cref{app:distance-analysis}).

\section{Related work}
\label{sec:related-work}

\paragraph{Probabilistic interpretations of MLMs.}
In one of the earliest works about sampling from MLMs, \citet{wang-cho} propose to use unary conditionals to sample sentences. 
Recently \citet{yamakoshi-etal-2022-probing} highlight that, while this approach only constitutes a pseudo-Gibbs sampler, the act of re-sampling positions uniformly at random guarantees that the resulting Markov chain has a unique, stationary distribution~\citep{bengio-2013,bengio-2014}.
Alternatively, \citet{goyal2021mrf} propose defining an MRF from the MLM's unary conditionals, and sample from this via Metropolis-Hastings. 
Concurrently, \citet{young2023compatibility} conduct an empirical study of the compatibility of \bert's conditionals.

\paragraph{Compatible distributions.}
The statistics community has long studied the problem of assessing the compatibility of a set of conditionals \citep{arnold1989compatible,gelman1993,DBLP:journals/ma/WangK10,song2010}. \citet{arnold1998algorithm} and \citet{ARNOLD2002231} explore algorithms for reconstructing near-compatible joints from incompatible conditionals, which we leverage in our work. 
\citet{besag1974conditionals} also explores this problem, and defines a procedure (viz., eq. \ref{eq:hcb}) for doing so when the joint distribution is strictly positive and the conditionals are compatible. \citet{lowd2012} apply a version of HCB to derive Markov networks from incompatible dependency networks \citep{heckerman2000dependency}.

\section{Conclusion}
In this paper, we studied four different methods for deriving an explicit joint distributions from MLMs, focusing in the pairwise language model setting where it is possible to compute exact distributional properties.
We find that the Arnold--Gokhale (AG) approach, which finds a joint whose conditionals are closest to the unary conditionals of an MLM, works best.
Indeed, our results indicate that said conditionals can attain lower perplexity than the unary conditionals of the original MLM. It would be interesting to explore whether explicitly regularizing the conditionals to be compatible during MLM training would lead to better modeling of the distribution of language.

\section{Limitations}
Our study illuminates the deficiencies of the MRF approach and applies statistically-motivated approaches to craft more performant probabilistic models. However, it is admittedly not clear how these insights can  immediately be applied to improve downstream NLP tasks. We focused on models over pairwise tokens in order to avoid sampling and work with exact distributions for the various approaches (MRF, HCB, AG). However this limits the generality of our approach (e.g., we cannot score full sentences).
We nonetheless believe that our empirical study is interesting on its own and suggests new paths for developing efficient and faithful MLMs.

\section*{Ethics Statement}
We foresee no ethical concerns with this work.

\section*{Acknowledgements}
We thank the anonymous reviewers for their helpful comments.
This  research is supported in part by funds from the MLA@CSAIL initiative and MIT-IBM Watson AI lab.
LTH acknowledges support from the Michael Athans fellowship fund.

\bibliography{custom}

\begin{thebibliography}{25}
\expandafter\ifx\csname natexlab\endcsname\relax\def\natexlab#1{#1}\fi

\bibitem[{Arnold et~al.(2002)Arnold, Castillo, and Sarabia}]{ARNOLD2002231}
Barry~C. Arnold, Enrique Castillo, and José~María Sarabia. 2002.
\newblock \href {https://doi.org/10.1016/S0167-9473(01)00111-6} {Exact and near
  compatibility of discrete conditional distributions}.
\newblock \emph{Computational Statistics \& Data Analysis}, 40(2):231--252.

\bibitem[{Arnold and Gokhale(1998)}]{arnold1998algorithm}
Barry~C. Arnold and Dattaprabhakar~V. Gokhale. 1998.
\newblock \href {https://doi.org/doi.org/10.1007/BF02565119} {Distributions
  most nearly compatible with given families of conditional distributions}.
\newblock \emph{Test}, 7(2):377--390.

\bibitem[{Arnold and Press(1989)}]{arnold1989compatible}
Barry~C. Arnold and James~S. Press. 1989.
\newblock \href {https://doi.org/10.2307/2289858} {Compatible conditional
  distributions}.
\newblock \emph{Journal of the American Statistical Association},
  84(405):152--156.

\bibitem[{Bengio et~al.(2014)Bengio, Thibodeau-Laufer, Alain, and
  Yosinski}]{bengio-2014}
Yoshua Bengio, \'{E}ric Thibodeau-Laufer, Guillaume Alain, and Jason Yosinski.
  2014.
\newblock \href {https://proceedings.mlr.press/v32/bengio14.html} {Deep
  generative stochastic networks trainable by backprop}.
\newblock In \emph{Proceedings of the 31st International Conference on Machine
  Learning}, volume~32 of \emph{Proceedings of Machine Learning Research},
  pages 226--234, Bejing, China. PMLR.

\bibitem[{Bengio et~al.(2013)Bengio, Yao, Alain, and Vincent}]{bengio-2013}
Yoshua Bengio, Li~Yao, Guillaume Alain, and Pascal Vincent. 2013.
\newblock \href
  {https://doi.org/https://dl.acm.org/doi/10.5555/2999611.2999712} {Generalized
  denoising auto-encoders as generative models}.
\newblock In \emph{Proceedings of the 26th International Conference on Neural
  Information Processing Systems}, NIPS, page 899–907, Red Hook, New York,
  USA. Curran Associates Inc.

\bibitem[{Besag(1974)}]{besag1974conditionals}
Julian Besag. 1974.
\newblock \href {https://www.jstor.org/stable/2984812} {Spatial interaction and
  the statistical analysis of lattice systems}.
\newblock \emph{Journal of the Royal Statistical Society}, 36(2):192--236.

\bibitem[{Bowman et~al.(2015)Bowman, Angeli, Potts, and Manning}]{snli}
Samuel~R. Bowman, Gabor Angeli, Christopher Potts, and Christopher~D. Manning.
  2015.
\newblock \href {https://doi.org/10.18653/v1/D15-1075} {A large annotated
  corpus for learning natural language inference}.
\newblock In \emph{Proceedings of the 2015 Conference on Empirical Methods in
  Natural Language Processing}, pages 632--642, Lisbon, Portugal. Association
  for Computational Linguistics.

\bibitem[{Chen and Goodman(1998)}]{CHEN1998}
Stanley~F. Chen and Joshua Goodman. 1998.
\newblock \href
  {https://people.eecs.berkeley.edu/~klein/cs294-5/chen_goodman.pdf} {An
  empirical study of smoothing techniques for language modeling}.
\newblock Technical report, Harvard University.

\bibitem[{Devlin et~al.(2019)Devlin, Chang, Lee, and
  Toutanova}]{devlin2019bert}
Jacob Devlin, Ming-Wei Chang, Kenton Lee, and Kristina Toutanova. 2019.
\newblock \href {https://doi.org/10.18653/v1/N19-1423} {{BERT}: Pre-training of
  deep bidirectional transformers for language understanding}.
\newblock In \emph{Proceedings of the 2019 Conference of the North {A}merican
  Chapter of the Association for Computational Linguistics: Human Language
  Technologies, Volume 1 (Long and Short Papers)}, pages 4171--4186,
  Minneapolis, Minnesota, USA. Association for Computational Linguistics.

\bibitem[{Du et~al.(2022)Du, Torroba~Hennigen, Pimentel, Meister, Eisner, and
  Cotterell}]{leo2023measure-theory}
Li~Du, Lucas Torroba~Hennigen, Tiago Pimentel, Clara Meister, Jason Eisner, and
  Ryan Cotterell. 2022.
\newblock \href {https://doi.org/10.48550/ARXIV.2212.10502} {A
  measure-theoretic characterization of tight language models}.

\bibitem[{Gelman and Speed(1993)}]{gelman1993}
Andrew Gelman and Terence~P. Speed. 1993.
\newblock \href
  {https://doi.org/https://doi.org/10.1111/j.2517-6161.1993.tb01477.x}
  {Characterizing a joint probability distribution by conditionals}.
\newblock \emph{Journal of the Royal Statistical Society}, 55(1):185--188.

\bibitem[{Ghazvininejad et~al.(2019)Ghazvininejad, Levy, Liu, and
  Zettlemoyer}]{ghazvininejad2019maskpredict}
Marjan Ghazvininejad, Omer Levy, Yinhan Liu, and Luke Zettlemoyer. 2019.
\newblock \href {https://doi.org/10.18653/v1/D19-1633} {Mask-{Predict}:
  {Parallel} decoding of conditional masked language models}.
\newblock In \emph{Proceedings of the 2019 {Conference} on {Empirical}
  {Methods} in {Natural} {Language} {Processing} and the 9th {International}
  {Joint} {Conference} on {Natural} {Language} {Processing}
  ({EMNLP}-{IJCNLP})}, pages 6112--6121, Hong Kong, China. Association for
  Computational Linguistics.

\bibitem[{Goyal et~al.(2022)Goyal, Dyer, and Berg-Kirkpatrick}]{goyal2021mrf}
Kartik Goyal, Chris Dyer, and Taylor Berg-Kirkpatrick. 2022.
\newblock \href {https://openreview.net/forum?id=6PvWo1kEvlT} {Exposing the
  implicit energy networks behind masked language models via
  {Metropolis}--{Hastings}}.
\newblock In \emph{International Conference on Learning Representations}.

\bibitem[{He et~al.(2021)He, Liu, Gao, and Chen}]{he2021deberta}
Pengcheng He, Xiaodong Liu, Jianfeng Gao, and Weizhu Chen. 2021.
\newblock \href {https://openreview.net/forum?id=XPZIaotutsD} {{DeBERTa:
  Decoding-enhanced BERT with Disentangled Attention}}.
\newblock In \emph{International Conference on Learning Representations}.

\bibitem[{Heckerman et~al.(2000)Heckerman, Chickering, Meek, Rounthwaite, and
  Kadie}]{heckerman2000dependency}
David Heckerman, Max Chickering, Chris Meek, Robert Rounthwaite, and Carl
  Kadie. 2000.
\newblock \href {https://www.jmlr.org/papers/v1/heckerman00a.html} {Dependency
  networks for inference, collaborative filtering, and data visualization}.
\newblock \emph{Journal of Machine Learning Research}, 1:49--75.

\bibitem[{Liu et~al.(2019)Liu, Ott, Goyal, Du, Joshi, Chen, Levy, Lewis,
  Zettlemoyer, and Stoyanov}]{liu2019roberta}
Yinhan Liu, Myle Ott, Naman Goyal, Jingfei Du, Mandar Joshi, Danqi Chen, Omer
  Levy, Mike Lewis, Luke Zettlemoyer, and Veselin Stoyanov. 2019.
\newblock \href {http://arxiv.org/abs/1907.11692} {{RoBERTa}: A robustly
  optimized {BERT} pretraining approach}.
\newblock \emph{CoRR}.

\bibitem[{Lowd(2012)}]{lowd2012}
Daniel Lowd. 2012.
\newblock \href {https://arxiv.org/abs/1210.4896} {Closed-form learning of
  markov networks from dependency networks}.
\newblock In \emph{Proceedings of the 28th Conference on Uncertainty in
  Artificial Intelligence}, pages 533--542, Catalina Island, California, USA.
  Association for Uncertainity in Artificial Intelligence.

\bibitem[{Narayan et~al.(2018)Narayan, Cohen, and Lapata}]{xsum}
Shashi Narayan, Shay~B. Cohen, and Mirella Lapata. 2018.
\newblock \href {https://doi.org/10.18653/v1/D18-1206} {Don't give me the
  details, just the summary! {T}opic-aware convolutional neural networks for
  extreme summarization}.
\newblock In \emph{Proceedings of the 2018 Conference on Empirical Methods in
  Natural Language Processing}, pages 1797--1807, Brussels, Belgium.
  Association for Computational Linguistics.

\bibitem[{Salazar et~al.(2020)Salazar, Liang, Nguyen, and
  Kirchhoff}]{salazar2020bertscore}
Julian Salazar, Davis Liang, Toan~Q. Nguyen, and Katrin Kirchhoff. 2020.
\newblock \href {https://doi.org/10.18653/v1/2020.acl-main.240} {Masked
  language model scoring}.
\newblock In \emph{Proceedings of the 58th Annual Meeting of the Association
  for Computational Linguistics}, pages 2699--2712, Online. Association for
  Computational Linguistics.

\bibitem[{Song et~al.(2010)Song, Li, Chen, Jiang, and Kuo}]{song2010}
Chwan-Chin Song, Lung-An Li, Chong-Hong Chen, Thomas~J. Jiang, and Kun-Lin Kuo.
  2010.
\newblock \href {https://www.jstor.org/stable/24308999} {Compatibility of
  finite discrete conditional distributions}.
\newblock \emph{Statistica Sinica}, 20(1):423--440.

\bibitem[{Wang and Cho(2019)}]{wang-cho}
Alex Wang and Kyunghyun Cho. 2019.
\newblock \href {https://doi.org/10.18653/v1/W19-2304} {{BERT} has a mouth, and
  it must speak: {BERT} as a {M}arkov random field language model}.
\newblock In \emph{Proceedings of the Workshop on Methods for Optimizing and
  Evaluating Neural Language Generation}, pages 30--36, Minneapolis, Minnesota,
  USA. Association for Computational Linguistics.

\bibitem[{Wang and Kuo(2010)}]{DBLP:journals/ma/WangK10}
Yuchung~J. Wang and Kun{-}Lin Kuo. 2010.
\newblock \href {https://doi.org/https://doi.org/10.1016/j.jmva.2009.07.007}
  {Compatibility of discrete conditional distributions with structural zeros}.
\newblock \emph{Journal of Multivariate Analysis}, 101(1):191--199.

\bibitem[{Wolf et~al.(2020)Wolf, Debut, Sanh, Chaumond, Delangue, Moi, Cistac,
  Rault, Louf, Funtowicz, Davison, Shleifer, von Platen, Ma, Jernite, Plu, Xu,
  Le~Scao, Gugger, Drame, Lhoest, and Rush}]{huggingface}
Thomas Wolf, Lysandre Debut, Victor Sanh, Julien Chaumond, Clement Delangue,
  Anthony Moi, Pierric Cistac, Tim Rault, Remi Louf, Morgan Funtowicz, Joe
  Davison, Sam Shleifer, Patrick von Platen, Clara Ma, Yacine Jernite, Julien
  Plu, Canwen Xu, Teven Le~Scao, Sylvain Gugger, Mariama Drame, Quentin Lhoest,
  and Alexander Rush. 2020.
\newblock \href {https://doi.org/10.18653/v1/2020.emnlp-demos.6} {Transformers:
  State-of-the-art natural language processing}.
\newblock In \emph{Proceedings of the 2020 Conference on Empirical Methods in
  Natural Language Processing: System Demonstrations}, pages 38--45, Online.
  Association for Computational Linguistics.

\bibitem[{Yamakoshi et~al.(2022)Yamakoshi, Griffiths, and
  Hawkins}]{yamakoshi-etal-2022-probing}
Takateru Yamakoshi, Thomas Griffiths, and Robert Hawkins. 2022.
\newblock \href {https://doi.org/10.18653/v1/2022.findings-acl.314} {Probing
  {BERT}'s priors with serial reproduction chains}.
\newblock In \emph{Findings of the Association for Computational Linguistics:
  ACL 2022}, pages 3977--3992, Dublin, Ireland. Association for Computational
  Linguistics.

\bibitem[{Young and You(2023)}]{young2023compatibility}
Tom Young and Yang You. 2023.
\newblock \href {https://doi.org/10.48550/ARXIV.2301.00068} {On the
  inconsistencies of conditionals learned by masked language models}.

\end{thebibliography}
\bibliographystyle{acl_natbib}

\clearpage
\onecolumn

\appendix

\section{MLMs as learning conditional marginals}
\label{app:mlm-objective-interpretation}

One can show that the MLM training objective corresponds to learning to approximate the conditional marginals of language, i.e., the (single-position) marginals of language when we condition on any particular set of positions.
More formally, consider an MLM parameterized by a vector $\vtheta \in \Theta$ and some distribution $\mu(\cdot)$ over positions to mask $S \subseteq [T]$.
Then the MLM learning objective is given by:
\begin{align*}
    \hat{\vtheta} = \argsup_\vtheta \Expect_{S \sim \mu(\cdot)} \Expect_{\vw \sim p(\cdot)} \left[ \frac{1}{\SetSize{S}} \sum_{t \in S} \log q_{t | \cS} (w_t \mid \vwcS; \vtheta) \right], 
\end{align*}
where $p(\cdot)$ denotes the true data distribution. Analogously, let $p_{S|\cS}(\cdot \mid \vwcS)$ and $p_{\cS}(\cdot)$ denote the conditionals and marginals of the data distribution, respectively. 
Then the above can be rewritten as:
\begin{align*}
    \hat{\vtheta} &= \argsup_\vtheta \Expect_{S \sim \mu(\cdot)} \Expect_{\vw_{\cS} \sim p_{\cS}(\cdot)} \left[ \frac{1}{\SetSize{S}} \sum_{t \in S} \Expect_{\vw_S \sim p_{S|\cS}(\cdot)} \left[ \log q_{t | \cS} (w_t \mid \vwcS; \vtheta) \right] \right]  \\
    &= \arginf_\vtheta \Expect_{S \sim \mu(\cdot)} \Expect_{\vw_{\cS} \sim p_{\cS}(\cdot)} \left[ \frac{1}{\SetSize{S}} \sum_{t \in S} \kl{p_{t | \cS}(\cdot \mid \vwcS)}{q_{t | \cS}(\cdot \mid \vwcS; \vtheta)}  \right],
\end{align*}
Thus, we can interpret MLM training as learning to approximate the conditional marginals of language, i.e., $\forall S \subseteq [T]$ and $\forall t \in S$, in the limit we would expect that, for any observed context $\vwcS$, we have $q_{t \mid \cS}(\cdot \mid \vwcS) \approx p_{t \mid \cS}(\cdot \mid \vwcS)$.

\section{Unfaithful MRFs}
\label{app:unfaithful-mrf}

Here we show that even if the unary conditionals used in the MRF construction are compatible~\citep{arnold1989compatible},
the unary conditionals of the probabilistic model implied by the MRF construction can deviate (in the KL sense) from the true conditionals.
This is important because (i) it suggests that we might do better (at least in terms of U-PPL) by simply sticking to the conditionals learned by MLM, and (ii) this is not the case for either the HCB or the AG constructions, i.e., if we started with the correct conditionals, HCB and AG's joint would be compatible with the MLM.
Formally,

\begin{proposition} 
Let $w_1, w_2 \in \mathcal{V}$ and further let $p_{1|2}(\cdot \mid w_2), p_{2|1}(\cdot \mid w_1)$ be the true (i.e., population) unary conditional distributions. Define an MRF as
\begin{align*}
    q_{1,2}(w_1, w_2) \propto p_{1|2}(w_1 \mid w_2) \, p_{2|1}(w_2 \mid w_1),
\end{align*}
and let $q_{1|2}(\cdot \mid w_2), q_{2|1}(\cdot \mid w_1)$ be the conditionals derived from the MRF. Then there exists $p_{1|2}, p_{2|1}$ such that 
\begin{align*}
    \kl{p_{1|2}(\cdot \mid w_2)&}{q_{1|2}(\cdot \mid w_2)} > 0.
\end{align*}
\end{proposition}

\begin{proof}

Let $w_2 \in \calV$ be arbitrary. We then have:
\begin{align*}
    q_{1|2}(w_1 \mid w_2) = \frac{p_{1|2}(w_1 \mid w_2) \, p_{2|1}(w_2 \mid w_1)}{\sum_{w' \in \calV} p_{1|2}(w' \mid w_2) \, p_{2|1}(w_2 \mid w')}
\end{align*}
Now, consider the KL between the true unary conditionals and the MRF unary conditionals:
\begin{align*}
    \kl{p_{1|2}(\cdot \mid w_2)&}{q_{1|2}(\cdot \mid w_2)} = \sum_{w \in \calV} p_{1|2}(w \mid w_2) \log \frac{p_{1|2}(w \mid w_2)}{q_{1|2}(w \mid w_2)} \\
    &= \sum_{w \in \calV} p_{1|2}(w \mid w_2) \log \frac{\sum_{w' \in \calV} p_{1|2}(w' \mid w_2) \, p_{2|1}(w_2 \mid w')}{p_{2|1}(w_2 \mid w)} \\
    &= \log \expect_{w \sim p_{1|2}(\cdot \mid w_2)} [ p_{2|1}(w_2 \mid w) ] - \expect_{w \sim p_{1|2}(\cdot \mid w_2)} [ \log p_{2|1}(w_2 \mid w)]
\end{align*}
This term is the Jensen gap, and in general it can be non-zero.
To see this, suppose $\calV = \{a, b\}$ and consider the joint
\begin{align*}
p_{1,2}(w_1, w_2) =
\begin{cases}
\frac{97}{100}                  & w_1, w_2 = a \\
\frac{1}{100}                  & \text{otherwise}
\end{cases}
\end{align*}
with corresponding conditionals $p_{2|1}(x \mid b) = p_{1|2}(x \mid b) = \frac{1}{2}$ for all $x \in \calV$ and
\begin{align*}
p_{2|1}(x \mid a) = p_{1|2}(x \mid a) = \begin{cases}
    \frac{97}{98}        &  x = a \\
    \frac{1}{98}         &  x = b \\
\end{cases}
\end{align*}
Now, take $w_2 = b$. We then have
\begin{align*}
    \kl{p_{1|2}(&\cdot \mid b)}{q_{1|2}(\cdot \mid b)}  \\
    &= \log \expect_{w \sim p_{1|2}(\cdot \mid b)} [ p_{2|1}(b \mid w) ] - \expect_{w \sim p_{1|2}(\cdot \mid b)} [ \log p_{2|1}(b \mid w)] \\
    &= \log \left( \frac{1}{2}  \left(\frac{1}{98} + \frac{1}{2} \right) \right) - \frac{1}{2} \left( \log \frac{1}{98} + \log \frac{1}{2} \right) \\
    &= \log \left( \frac{1}{196} + \frac{1}{4} \right) - \frac{1}{2} \left( \log \frac{1}{196} \right) \approx 1.27
\end{align*}
which demonstrates that the KL can be non-zero.

\end{proof}

\section{Arnold--Gokhale algorithm}
\label{app:ag-algorithm}

\citet{arnold1998algorithm} study the problem of finding a near-compatible joint from unary conditionals,
and provide and algorithm for the case of $|S| = 2$.
The algorithm initializes the starting pairwise distribution $\qAGiter{1}_{a,b|\cS}(\cdot, \cdot \mid \vwcS)$  to be uniform, and performs the following update until convergence:
\begin{align}
    \qAGiter{t+1}_{a,b|\cS}(w_a, w_b \mid \vwcS) \propto \frac{q_{a|b,\cS}(w_a \mid w_b, \vwcS) + q_{b|a,\cS}(w_b \mid w_a, \vwcS)}{\left(\qAGiter{t}_{a|\cS}(w_a \mid \vwcS)\right)^{-1} + \left(\qAGiter{t}_{b|\cS}(w_b \mid \vwcS)\right)^{-1}}
    \label{eq:iter-update}.
\end{align}

\section{Qualitative example of MRF underperformance}
\label{app:mrf-underperformance}

This example from SNLI qualitatively illustrates a case where both the unary and pairwise perplexities from the MRF underperforms the MLM: ``The \texttt{[MASK]$_1$} \texttt{[MASK]$_2$} at the casino'', where the tokens ``man is'' are masked. In this case, both MRFs assign virtually zero probability mass to the correct tokens, while the MLM assigns orders of magnitude more (around $0.2\%$ of the mass of the joint).
Upon inspection, this arises because $q_{2|1,\cS}(\text{is} \mid \text{man}) \approx 0.02$ and $q_{1|2,\cS}(\text{man} \mid \text{is}) \approx \num{2e-5}$, which makes the numerator of $q^\text{MRF}_{1,2|\cS}(\text{man}, \text{is})$ be $\approx 0$. The MRF could still assign high probability to this pair if the denominator is also $\approx 0$, but in this case we have  $q_{2|1,\cS}(\text{was} \mid \text{man}) \approx 0.33$ and $q_{1|2,\cS}(\text{man} \mid \text{was}) \approx 0.03$, which makes the denominator well above 0. This causes the  completion ``man is'' to have disproportionately little mass in the joint compared other to combinations (``man was'') that were ascribed more mass by \bert's unary conditionals.

\section{Token distance analysis}
\label{app:distance-analysis}

We also explore the effect of the distance between masked tokens on the pairwise negative log-likelihood (PNLL, lower is better; note this is equivalent to the log PPPL) of the joints built using the different approaches we considered.
We considered two different kinds of distance functions between tokens: (i) the absolute difference in the positions between the two masked tokens, and (ii) their syntactic distance (obtained by running a dependency parser on unmasked sentences).

We plot the results in \cref{fig:snli-distance-analysis} (SNLI) and \cref{fig:xsum-distance-analysis} (XSUM).
Note that the black bars denote the number of datapoints with that distance between the two masked tokens, where a syntactic distance of 0 means that the two masked tokens belong to the same word, whereas a token distance of 0 means that the two masked tokens are adjacent.
The graphs indicate that the language modeling performance improvement (compared to using the MLM joint) is most prominent when masked tokens are close together, which is probably because when the masked tokens are close together they are more likely to be dependent. In this case, AG tends to do best, HCB and MRF tend to do similarly, followed by MRF-L and, finally, the conditionally independent MLM, which follows the trends observed in the paper.

\begin{figure}
    \begin{minipage}{0.5\columnwidth}
        \centering
        \includegraphics[width=0.95\textwidth]{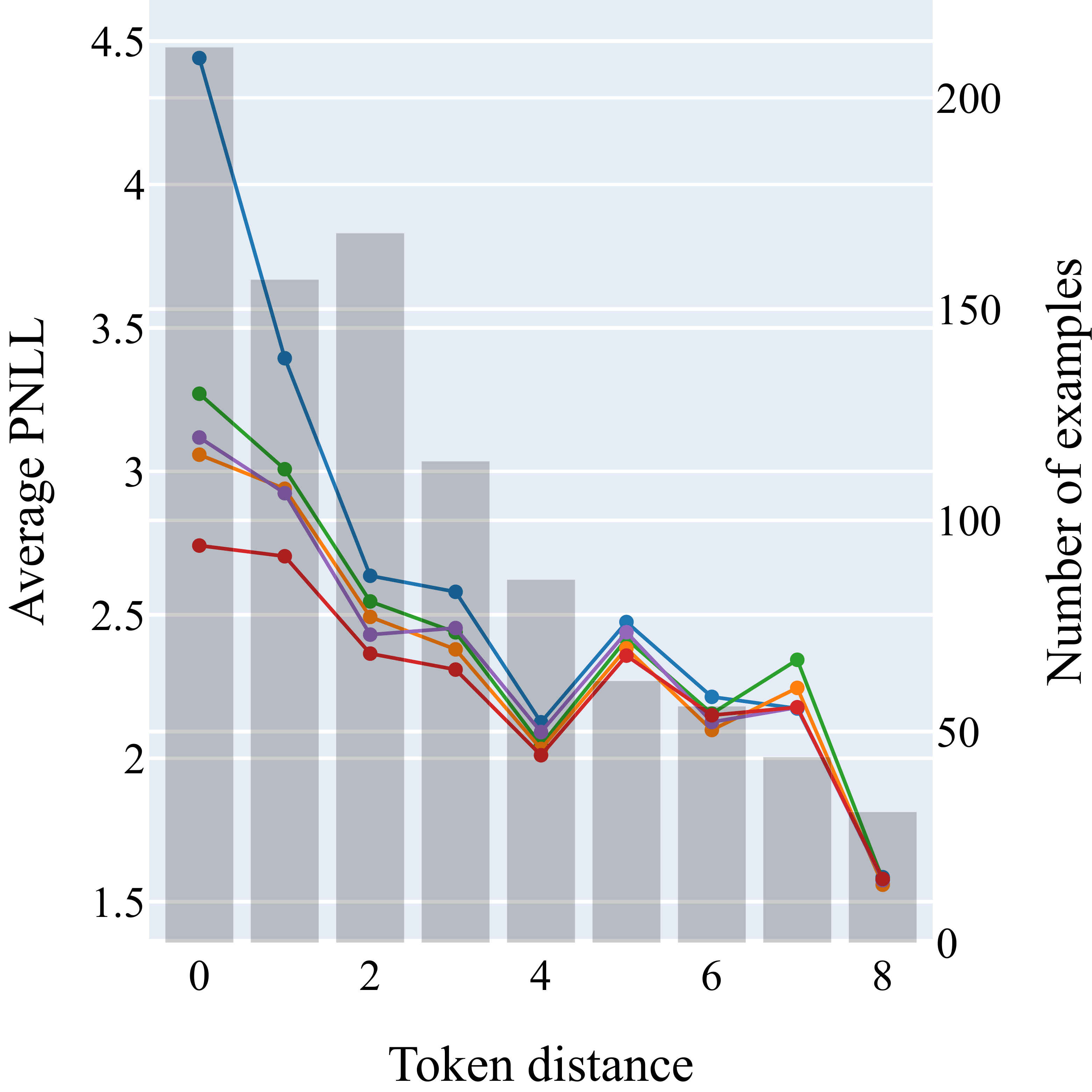}
    \end{minipage}%
    \begin{minipage}{0.5\columnwidth}
        \centering
        \includegraphics[width=0.95\textwidth]{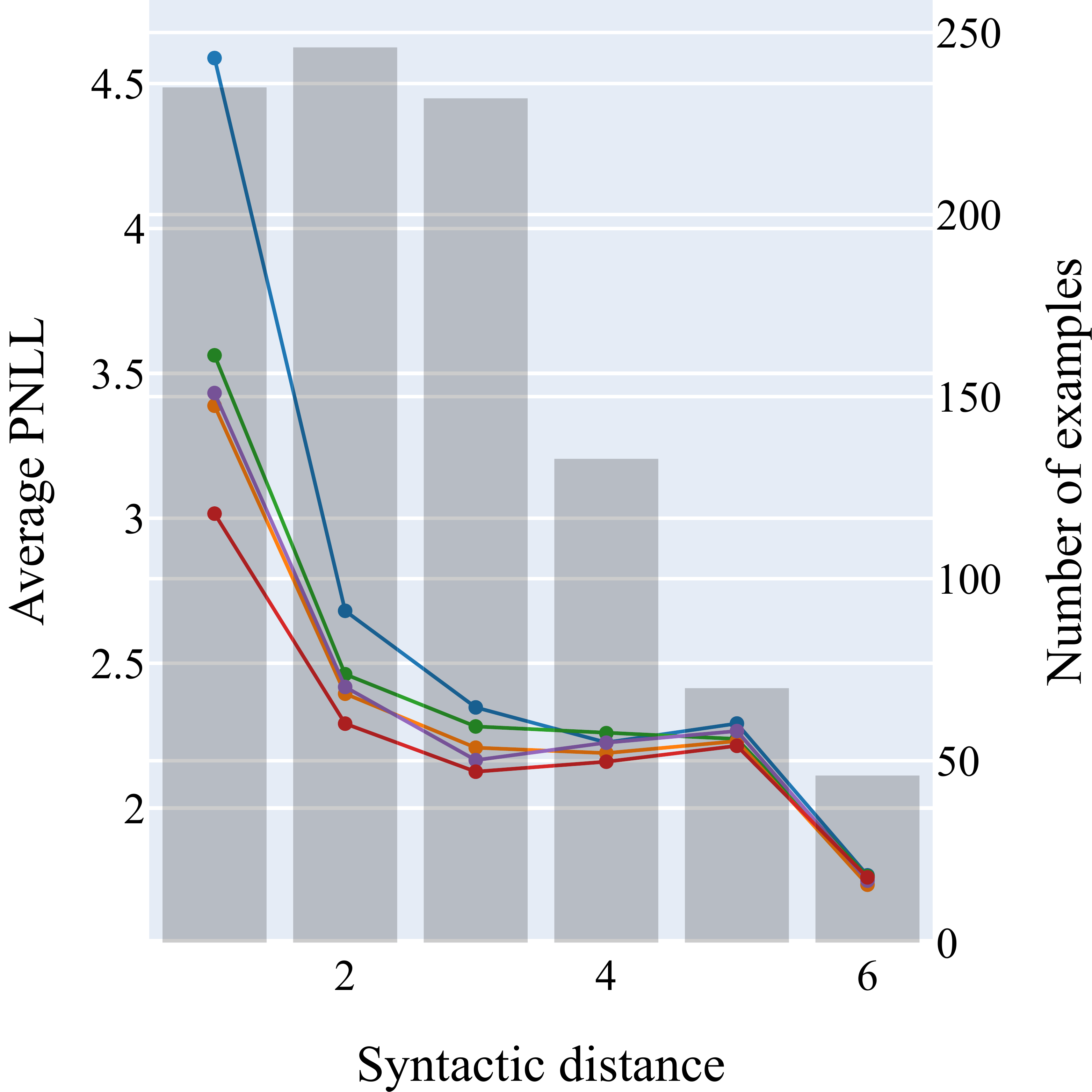}
    \end{minipage}
    \caption{Pairwise NLL (PNLL) as a function of the token and syntactic distance between masked positions for joints built using the methods: \mlmC, \mrfLogitC, \mrfC, \hcbC, \agC on SNLI~\citep{snli}. The gray bars represent the number of examples on the dataset that had that degree of separation.}
    \label{fig:snli-distance-analysis}
\end{figure}

\begin{figure}
    \begin{minipage}{0.5\columnwidth}
        \centering
        \includegraphics[width=0.95\textwidth]{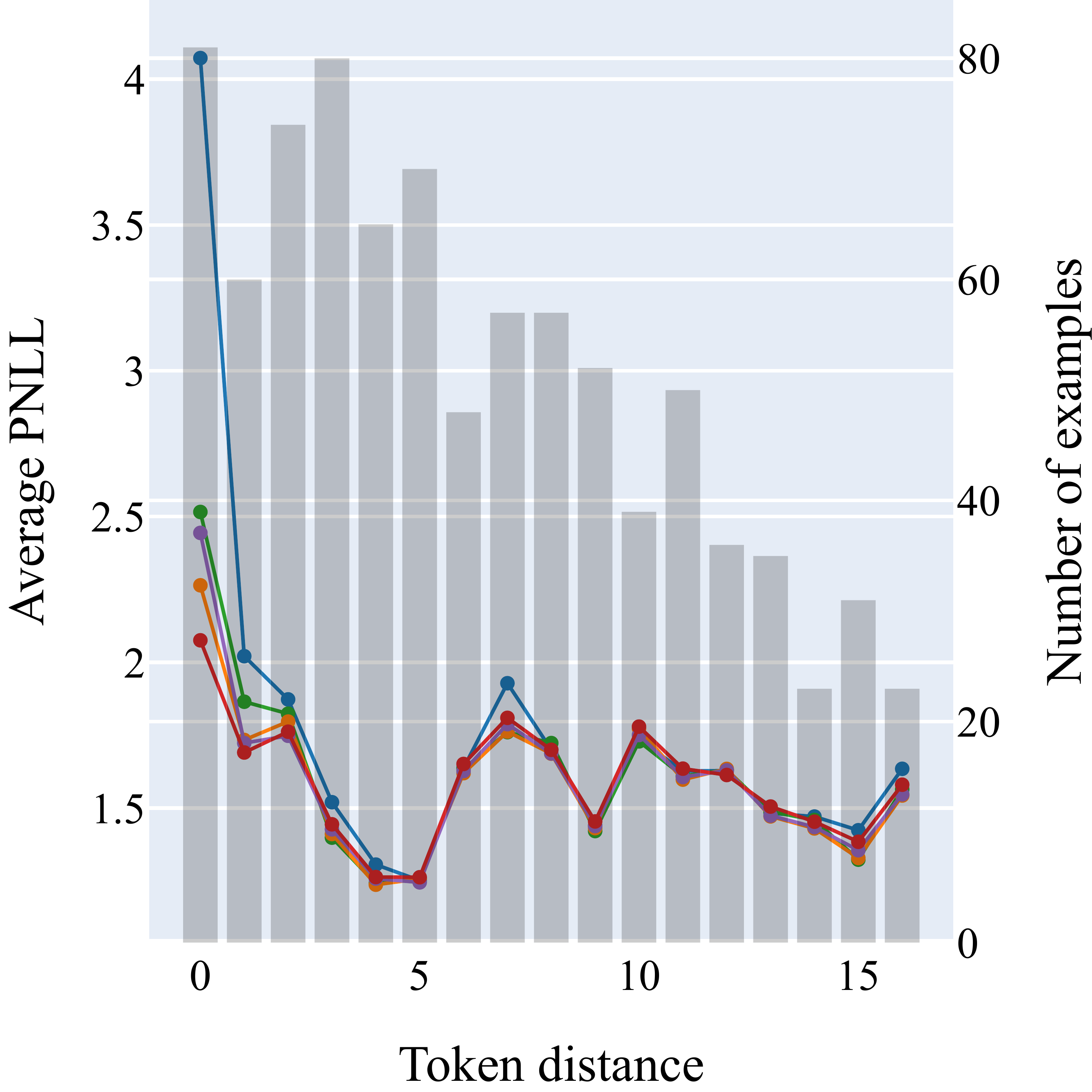}
    \end{minipage}%
    \begin{minipage}{0.5\columnwidth}
        \centering
        \includegraphics[width=0.95\textwidth]{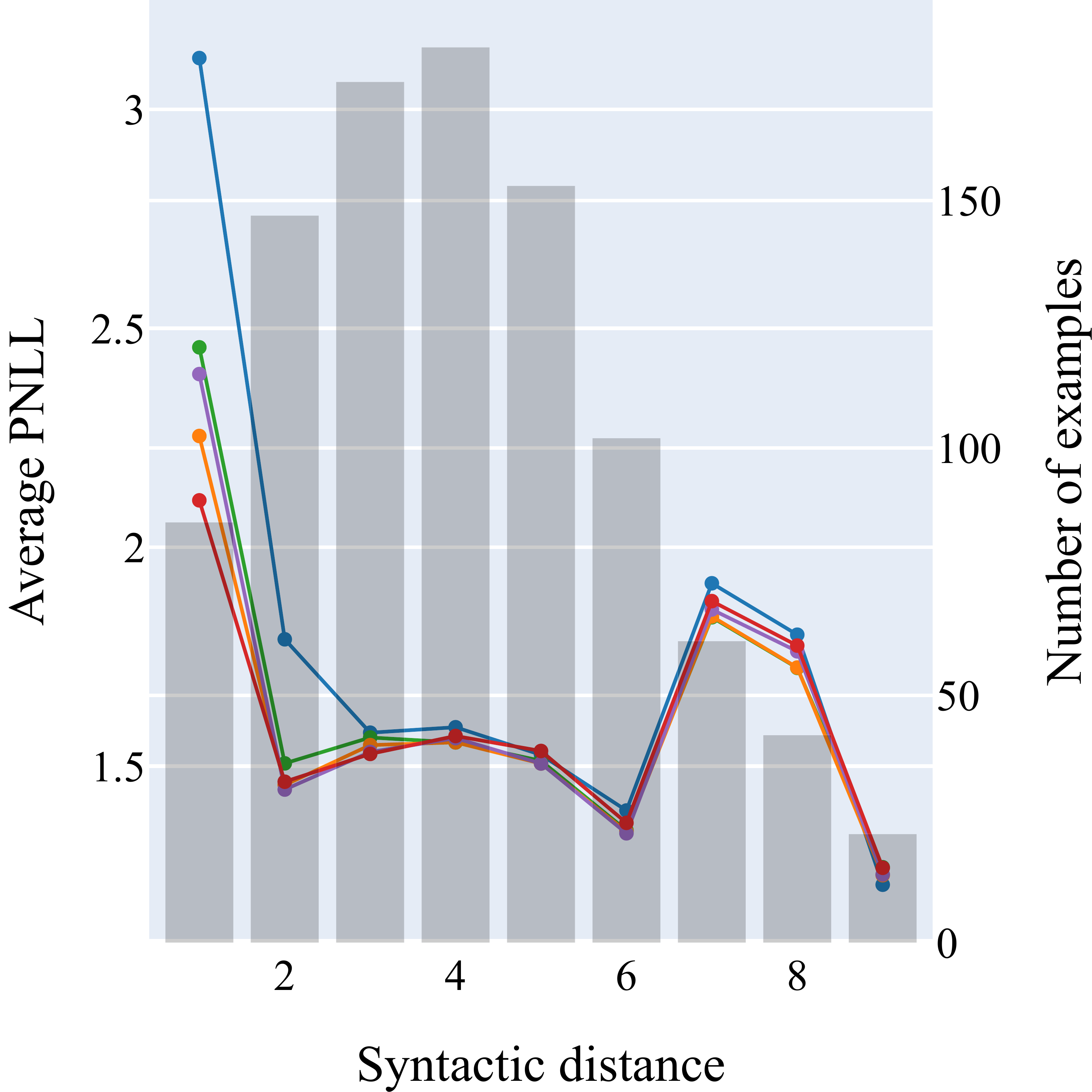}
    \end{minipage}
    \caption{Pairwise NLL (PNLL) as a function of the token and syntactic distance between masked positions for joints built using the methods: \mlmC, \mrfLogitC, \mrfC, \hcbC, \agC on XSUM~\citep{xsum}. The gray bars represent the number of examples on the dataset that had that degree of separation.}
    \label{fig:xsum-distance-analysis}
\end{figure}

\end{document}